\documentclass[nohyperref]{article}

\usepackage{microtype}
\usepackage{graphicx}
\usepackage{subfigure}
\usepackage{booktabs} 

\usepackage{hyperref}

\usepackage[accepted]{icml2023}

\usepackage{amsmath}
\usepackage{amssymb}
\usepackage{mathtools}
\usepackage{amsthm}
\usepackage{mathtools}
\newcommand{\defeq}{\vcentcolon=}
\usepackage[capitalize,noabbrev]{cleveref}
\usepackage{bbm}

\DeclareMathOperator{\Pa}{Pa}
\DeclareMathOperator{\De}{De}
\DeclareMathOperator{\Nd}{Nd}

\theoremstyle{plain}
\newtheorem{theorem}{Theorem}[section]
\newtheorem{proposition}[theorem]{Proposition}
\newtheorem{lemma}[theorem]{Lemma}
\newtheorem{corollary}[theorem]{Corollary}
\theoremstyle{definition}
\newtheorem{definition}[theorem]{Definition}

\theoremstyle{remark}

\theoremstyle{definition}
\newtheorem{example}{Example}[section]
	
\theoremstyle{plain}
    \newtheorem{innercustomgeneric}{\customgenericname}
    \providecommand{\customgenericname}{}
    \newcommand{\newcustomtheorem}[2]{%
      \newenvironment{#1}[1]
      {%
       \renewcommand\customgenericname{#2}%
       \renewcommand\theinnercustomgeneric{##1}%
       \innercustomgeneric
      }
      {\endinnercustomgeneric}
    }
    
    \newcustomtheorem{customthm}{Theorem}
    \newcustomtheorem{customlemma}{Lemma}
    \newcustomtheorem{customprop}{Proposition}
    \newcustomtheorem{customcorollary}{Corollary}

\def\EE{{\mathbb E}}    
\def\11{{\mathbf 1}}    

\DeclareSymbolFont{boldoperators}{OT1}{cmr}{bx}{n}
\SetSymbolFont{boldoperators}{bold}{OT1}{cmr}{bx}{n}

\usepackage[textsize=tiny]{todonotes}

\icmltitlerunning{Results on Counterfactual Invariance}

\begin{document}

\twocolumn[
\icmltitle{Results on Counterfactual Invariance}



\icmlsetsymbol{equal}{*}

\begin{icmlauthorlist}
\icmlauthor{Jake Fawkes}{yyy}
\icmlauthor{Robin J. Evans}{yyy}
\end{icmlauthorlist}

\icmlaffiliation{yyy}{Department of Statistics, University of Oxford}

\icmlcorrespondingauthor{Jake Fawkes}{jake.fawkes@st-hughs.ox.ac.uk}

\icmlkeywords{Machine Learning, ICML}

\vskip 0.3in
]



\printAffiliationsAndNotice{}  

\begin{abstract}
In this paper we provide a theoretical analysis of counterfactual invariance. We present a variety of existing definitions, study how they relate to each other and what their graphical implications are. We then turn to the current major question surrounding counterfactual invariance, how does it relate to conditional independence? We show that whilst counterfactual invariance implies conditional independence, conditional independence does not give any implications about the degree or likelihood of satisfying counterfactual invariance. Furthermore, we show that for discrete causal models, counterfactually invariant functions are often constrained to be functions of particular variables, or even constant. 
\end{abstract}
\vspace{-2em}
\section{Introduction} 

Causality has emerged as an important language to reason about generalisation and invariance in machine learning. These approaches to generalisation have involved --- but are not limited to --- modelling changing environments within causal graphs \citep{peters2016causal}, defining invariant predictors in terms of causal graphs \citep{buhlmann2020invariance} and applying causality to predict invariant conditional distributions \citep{magliacane2018domain}.

 \textit{Counterfactual Invariance} has arisen as a promising new causal definition of invariance  \citep{veitch2021counterfactual}. The idea is that we should aim for predictors which lead to the same outcome had we, contrary to fact, intervened on a spurious part of our data. For example, if we change the name of an actor in a movie review but keep everything else the same, the sentiment of the review should not change. Therefore we want the sentiment analysis model to be invariant to counterfactual perturbations of spurious features.  

The goal of this paper is to analyse a variety of definitions of counterfactual invariance and examine what implications they have in causal models. We look at the definition introduced in \citet{veitch2021counterfactual} and an alternative based upon distributional equality from \citet{quinzan2022learning}, considering how they relate to each other and what they imply in terms of graphical structure. 

We then turn to the central question of the relationship between invariance and conditional independence. First, we give a set of independences implied by counterfactual invariance that encompass many existing results in the literature. We then turn to the reverse implication, asking if these independences imply that a given variable or function satisfies counterfactual invariance. To this our answer is negative, and we show that independences do not bound the degree of counterfactual invariance, and that counterfactual invariance remains in some sense `unlikely,' even if the independences hold. This raises questions about what passing conditional independence tests achieves in terms of counterfactual invariance, as well as how best to discuss the benefit conditional independence brings. 

Finally we prove that for almost all discrete causal models the only counterfactually invariant functions are functions whose inputs do not causally depend on the intervening variable. This implies that if all variables causally depend on the intervening variable then almost always the only counterfactually invariant functions are constant. This has implications for counterfactual fairness, as in this context many have argued that almost all variables will depend on sensitive attributes \citep{kusner2017counterfactual}. These results suggest that in these cases the only counterfactually fair functions will be constant.   
\vspace{-0.5em}
\section{Background}\label{sec:background}

\paragraph{Notations} Throughout we will use $Y$ to denote the variable we are assessing counterfactual invariance in, $Z$ to be the variable we are intervening on, and $X$ to be the remaining covariates. We will use $\mathcal{Y},\mathcal{Z}$ and $\mathcal{X}$ to denote the domains of the variables. 

In reference to a causal graph we will use $\Pa(Y)$ to denote the causal parents of $Y$, $\De(Z)$ for the causal descendants of $Z$, i.e.~those variables, including $Z$, that can be reached from $Z$ via a directed path, and $\Nd(Z)$ for the variables which are not descendants of $Z$.
\vspace{2em}
\paragraph{Structural Causal Models (SCM)}\citep{peters2017elements} A structural causal model is defined as a pair $\mathbbm{C} \coloneqq(\mathbf{S},P_{\mathbf{N}})$ where $\mathbf{S}$ is a collection of $d$ structural assignments:
\begin{align*}
    X_j \coloneqq f_j (Pa(X_j),N_j), \hspace{0.2cm} j=1,...,d,
\end{align*}
where $N_j$ is the noise variable for generating $X_j$ and $P_{\mathbf{N}}=P_{N_1,...,N_d}$ is a distribution over the noise variables. In this we will take all the noise variables to be independent\footnote{In fact most results in this paper do not require full independence and can be adapted to less restrictive counterfactual models; see, for example, \citet{robins2010alternative}.} and will use $\mathcal{N}$ to refer to their domain. We define $Y(z)$ to be the variable $Y$ when we intervene in the structural causal models to set $Z=z$. Given the noise distribution, we may define the distribution over arbitrary events, even counterfactual ones, as just the sum over noise variables which permit such an event. For example:
\begin{align*}
    P(Y(z) = y, Y=y^{\prime}) = \sum_{\substack{\mathbf{n} \in \mathcal{S} }} P_{\mathbf{N}}(\mathbf{n})
\end{align*}
where $\mathcal{S} = \{ \mathbf{n} \in \mathcal{N} : Y(z,\mathbf{n}) = y\;  \&\; Y(\mathbf{n}) = y^{\prime} \}$. 
\subsection{Definitions of Counterfactual Invariance}
We now present a variety of different definitions of counterfactual invariance, inspired existing definitions in the literature. The first is in terms of almost sure equality and is based upon \citet{veitch2021counterfactual}:

\begin{definition}
A variable $Y$ satisfies \textbf{almost sure counterfactual invariance} (\textit{a.s.-CI}) with respect to $Z$ if: 
$$Y(z) \overset{a.s.}{=} Y(z^{\prime}) \text{  for all $z,z^{\prime}$.}$$
\end{definition}
Counterfactual invariance can also be defined in terms of distributional equality, similar to \citet{quinzan2022learning}:
\begin{definition}
We say a variable $Y$ satisfies \textbf{distributional counterfactually invariance} ($\mathcal{D}$-CI) conditional on some set of variables $W$, with respect to $Z$, if:
\begin{align*}
   P(Y(z)\!=\!y | W\!=\!w,Z\!=\!z) \!=\! P(Y(z^{\prime})\!=\!y | W\!=\!w,Z\!=\!z), \vspace{-10em}
\end{align*}
for all $z,z^{\prime}$ and for almost all $y,w$.
\end{definition}

Our definition deviates slightly from \citet{quinzan2022learning} as we enforce conditioning on the intervening variable. We do this to ensure the definition is truly counterfactual in the sense that we are always asking ``\textit{what would have happened had we intervened to set $Z=z^{\prime}$, given that in reality $Z=z$.}'' 
Finally we can define counterfactually invariant functions:
\begin{definition}[\citealt{veitch2021counterfactual}]
A function $f: \mathcal{X} \to \mathcal{Y}$ is \textbf{counterfactually invariant} (\textsc{\textit{$\mathcal{F}$-CI}}) if the variable $\hat{Y}  \coloneqq f(X)$ satisfies almost sure counterfactual invariance. That is:
\begin{align*}
    f(X(z)) \overset{a.s.}{=} f(X(z^{\prime})).
\end{align*}
\end{definition}

\paragraph{Relationships Between Definitions} We now give some basic relationships between the definitions, with proofs in the appendix:
\begin{lemma}\label{lem:almos_sure_relation_dist}
We have that \textit{a.s.}-CI implies $\mathcal{D}$-CI conditional on any set of variables, but $\mathcal{D}$-CI implies \textit{a.s.}-CI only if the conditioning set contains the outcome, $Y$.
\end{lemma}
This lemma allows us to equivalently define functional counterfactual invariance in terms of distributional counterfactual invariance, conditional on the function inputs. This is because conditioning on the inputs implicitly conditions on the value of the function. 
\begin{corollary}\label{lem:func_ci_equiv_dist}
    A function $f$ is $\mathcal{F}$-CI if and only if $\hat{Y} = f(X) $ is $\mathcal{D}$-CI conditional $X$.
\end{corollary}
This also shows that the definition presented in \citet{kusner2017counterfactual} is equivalent to that presented in \citet{veitch2021counterfactual}.

\section{Counterfactual Invariance and the Causal Graph}

We first look at how counterfactual invariance relates to the causal graph with the following lemma showing there can be graphical implications of counterfactual invariance:
\begin{lemma}\label{lem:cf_invariance_no_edge}
    If $Y$ is \textit{a.s.}-CI with respect to $Z$ then there is no edge from $Z$ to $Y$.
\end{lemma}
However, as the following example shows, counterfactual invariance does not lead to the absence of any other edges:

\begin{example}
Suppose we have the following causal model:
    \begin{align*}
        Z &= N_Z \\
        X &= 2 \mathbbm{1} \{Z=1\} + N_X \\
        Y &= \mathbbm{1} \{ \mathrm{mod}_2 (X) = 0 \} + N_Y,
    \end{align*}
where all noise variables are Bernoulli with parameter $\frac{1}{2}$. Under this model $Y$ is almost surely counterfactually invariant but the causal graph is $Z \to X \to Y$. Therefore counterfactual invariance does not imply the absence of any edge besides $Z \to Y$.
\end{example}

Moreover, counterfactual independence can be implied by graph structure alone. Namely when there is no direct causal path from the intervening variable to the outcome: 

\begin{lemma}[\citealt{kusner2017counterfactual}]\label{lem:non_descendants_imply_invariance}
  If $Y$ is not a causal descendant of $Z$ it is counterfactually invariant.
\end{lemma}

Therefore, in some cases, counterfactual invariance can be read directly from the graph.  

\section{Counterfactual Invariance and Conditional Independence}

We now turn to the relationship between counterfactual invariance and conditional independence. This has received a large amount of attention as it is one of the only testable implications of counterfactual invariance. In this section we first give a selection of conditional independences implied by the different forms of invariance before turning to the reverse implication, can conditional independence tell us anything about the degree of counterfactual invariance? Finally for discrete models we give a more exact classification of the counterfactually invariant functions in terms of $Nd(Z)$. 

\subsection{Invariance Implies Independence}\label{sec:indep_constraints}

We first give a selection of independences implied by counterfactual invariance. In order to do we first need the definition of a valid adjustment set:

\begin{definition}
    A set of variables $S$ is known as a valid adjustment set for the pair $(Z,Y)$ if:
    \begin{align*}
        P( Y(z)) = \int P(Y \mid s,z) P(s) ds.
    \end{align*}
\end{definition}
This leads us to the following result for distributional counterfactual invariance. The proof expands on that given in \citet{fawkes2022selection} and can be found in the appendix.

\begin{lemma}\label{lem:invar_implies_indep}
    Suppose $Y$ is distributionally counterfactual invariant conditional on some set $W$, such that there is some valid adjustment set $S \subset W$ for $(Z,Y)$. Then we have $Y \perp Z \mid S$.   
\end{lemma}

We can now apply the relationships between the definitions that were given in Section \ref{sec:background} to get a set of independences implied by the other forms of invariance. 

\begin{corollary}\label{cor:indep_implies_by_a.s_and_f}
The following holds;
\vspace{-0.6em}\begin{itemize}
\setlength\itemsep{-0.15em}
    \item if $Y$ is \textit{a.s.-CI} then $Y \perp Z \mid S$ for any valid adjustment set $S$ for $(Z,Y)$;
    \item  if $f$ is $\mathcal{F}$\textit{-CI} then $f(X) \perp Z \mid S$ for any valid adjustment set $S$ for $(Z,f(X))$
\vspace{-0.6em}\end{itemize}
\end{corollary}

For DAG models \citet{shpitser2012validity} showed that valid adjustment sets for $(Z,Y)$ can be classified as  sets $S$ which satisfy:
\vspace{-0.6em}\begin{itemize}
            \setlength\itemsep{-0.15em}
    \item $S$ blocks all non-directed paths from $Z$ to $Y$.
    \item $S$ does not contain any descendants of any nodes on a direct path from $Z$ to $Y$.
\vspace{-0.6em}\end{itemize}
If we apply this graphical characterisation alongside Corollary \ref{cor:indep_implies_by_a.s_and_f} we can re-derive all the conditional independence implications given in \citet{veitch2021counterfactual}. Further, for more general graphs the conditional independences implied by counterfactual invariance can now be read directly from the graph by finding the relevant adjustment sets. Finally, our proof does not require the any measurability results of additional variables as in \citet{veitch2021counterfactual}, so we may apply these results without having to rely on additional assumptions like discrete $Z$. 

\subsection{But Independence Does Not Imply Invariance}

The observational implications presented in the previous section have lead to arguments that when we would like a counterfactually invariant predictor, we should regularise for the conditional independences implied by it. Whilst it is true that variables not satisfying these independences cannot be counterfactually invariant, it is not clear if satisfying conditional independences can tell us anything about the degree of counterfactual invariance.  

In this section we study the question of how independence affects invariance. We begin with with an example that  shows the difficulty of achieving counterfactual invariance through independence alone. 

\begin{example}\label{ex: binary_response_DAG}
Suppose the only variables are $Z,Y$ which are binary and follow the DAG  $Z \to Y$. Now if we fix the noise variable $N_Y$ to a given value $n_Y$ we have that $Y$ is a deterministic function of $Z$ as:
\begin{align*}
    Y = f_Y(Z,n_Y).
\end{align*}
So as the noise, $N_Y$, varies it corresponds to $f$ switching between different deterministic functions from $\{0,1 \}$ to $\{0,1 \}$. Because of the finite domains it is possible to enumerate all functions, $f:\{0,1 \} \to \{0,1 \}$, as:
\begin{align*}
    f_0(z) &= 0  & f_1(z) &= 1 \\
    f_2(z) &= z & f_3(z) &= 1-z.
\end{align*}
This allows us to represent any possible structural causal model as distribution over the functions $\{ f_0,f_1,f_2,f_3 \}$ or more simply as a distribution, $P(\Tilde{N}_Y)$, over $\{0,1,2,3\}$ where when $\Tilde{N}_Y = i$ we have that:
\begin{align}\label{eq:reparameter_SCM}
    Y=\Tilde{f}_Y(Z,i) = f_i(Z)
\end{align}
for the new structural function $\Tilde{f}_Y$. This is known as the function response framework \citep{balke1994counterfactual}. Details of this construction for more general discrete causal models can be found in Section 3.4 of \citet{peters2017elements}.

Given an observational distribution, $P(Y,Z)$, over $Y,Z$ we can find all possible SCMs that could give rise to this observed distribution simply by finding all possible distributions of $\Tilde{N}_Y$
that give $P(Y\mid Z)$ under (\ref{eq:reparameter_SCM}). To do this let $p_{ij} =P(Y=i \mid Z=j)$. Now the set of possible distributions which comply with $P(Y,Z)$ can be written as probability vectors in a single free parameter $\lambda$ as:
\begin{align*}
    \mathbf{a}(\lambda) = \begin{pmatrix}
        0 \\
        1-p_{00}-p_{01} \\
        p_{00} \\
        p_{01} \\
    \end{pmatrix} +\lambda\begin{pmatrix}
        1 \\
        1 \\
        -1 \\
        -1 \\
    \end{pmatrix},
\end{align*}
where $\mathbf{a}(\lambda) = P(\Tilde{N}_Y=i)$ and $\lambda$ is constrained so that $\lambda_{\min}= {\max}\{0, p_{00}\!+\!p_{01}\!-\!1\}$ and $\lambda_{\max}= {\min}\{p_{00},p_{01}\}$.

We may now assess the degree of counterfactual invariance across the causal models with the following:
\begin{align*}
 P(Y(0) = Y(1)) &= \mathbf{a}(\lambda)_0+\mathbf{a}(\lambda)_1 \\
 &= 1 - p_{00} - p_{01} +2 \lambda,
\end{align*}
where we have almost sure counterfactual invariance precisely when this quantity is equal to 1. We can see this is only possible when $p_{00} = p_{01}$, recovering the independence requirement from Section \ref{sec:indep_constraints}. Letting $p_{00}= p_{01} = p$ we can see the set of observationally equivalent SCM is now given by:
\begin{align*}
    \mathbf{a}(\lambda) = \begin{pmatrix}
        0 \\
        1-2p \\
        p \\
        p \\
    \end{pmatrix} +\lambda\begin{pmatrix}
        1 \\
        1 \\
        -1 \\
        -1 \\
    \end{pmatrix},
\end{align*}
where almost sure counterfactual invariance occurs precisely when $\lambda = p$. However, in this case $\lambda$ can vary between $\lambda_{\min}= \max\{0, 2p-1 \}$ and $\lambda_{\max}= p$, giving a whole range of observationally equivalent SCMs which are not counterfactually invariant. Furthermore we can see that there is only one possible counterfactually invariant SCM amongst the infinite observationally equivalent models. 
\end{example}
This example demonstrates the difficulty of validating counterfactual invariance from observational data alone, even in the most simple of cases. Further, as we can embed versions of this example within more general structural causal models this leads to more general results:

\begin{proposition} \label{prop:indep_not_imply_a.s}
For almost sure counterfactual invariance we have that:
    \vspace{-0.6em}\begin{itemize}
            \setlength\itemsep{-0.15em}
        \item No set of independences imply any bound on $P(Y(z) = Y(z^{\prime}))$, even if $Y$ is binary.
        \item For discrete models, the counterfactually invariant structural causal models have Lebesgue measure zero within the models satisfying the set of independences in Corollary \ref{cor:indep_implies_by_a.s_and_f}. 
    \end{itemize}
\end{proposition}

The Lebesgue measure zero result is similar to existing results on the faithfulness assumption in causal inference. The faithfulness assumption is common in causal discovery as it states that any observed independences arise only from graphical structure. A common justification for making this assumption is that for discrete and Gaussian models it is violated on a set of Lebesgue measure zero \citep{meek1995strong}. Our result is similar to this, however in our case it advises against concluding counterfactual invariance from conditional independence alone. 

Due to the relationship between definitions this also gives us results for distributional and functional counterfactual invariance. We present the result for functional counterfactual invariance with the distributional result in Appendix \ref{App:distributional_inv}:

\begin{proposition} \label{prop:indep_not_imply_func}
    For non constant functions $f$ we have that:   
    \vspace{-2em}\begin{itemize}
        \setlength\itemsep{-0.1em}
        \item No independences imply $f$ is counterfactually invariant.
        \item For $f: \mathcal{X} \to \{0,1\}$ no independences imply any bound on $P(f(X(z)) = f(X(z^{\prime}))$.
    \end{itemize}
\end{proposition}
This suggests that whilst counterfactually invariant functions must satisfy certain conditional independences, these independences give little assurances as to the degree of counterfactual invariance. 

Moreover for discrete models we can more precisely characterize the set of counterfactually invariant functions:
\begin{proposition}\label{prop:classify_discrete_CI_F}
   For a given causal model let $\mathcal{F}^{CI}$ be the set of counterfactually invariant functions $f: \mathcal{X} \to \mathcal{Y}$. Supposing all the variables are discrete we have that for almost all structural causal models:
   \begin{align*}
      \mathcal{F}^{CI} = \{f : f \text{ is a function of $\Nd(Z)$ only.} \} 
   \end{align*}
Alternatively stated, the set of causal models that admit a counterfactually invariant function which takes a descendant of $Z$ as an input, has Lebesgue measure zero.
\end{proposition}

This suggests that when trying to build counterfactually invariant predictors from observational data alone it is almost always impossible to do better than Lemma \ref{lem:non_descendants_imply_invariance}, the original and most basic way to ensure counterfactual invariance provided by \citet{kusner2017counterfactual}.  

A further implication of this statement is that if every variable is a descendant of $Z$ then for almost all causal models the only counterfactually invariant functions are constant.  In the language of \citet{veitch2021counterfactual}, this means the variable $X_{Z}^{\perp}$ corresponding to the part of $X$ not causally affected by $Z$ is almost surely constant. Whilst these results are for discrete models it raises questions about which general causal models admit non-constant counterfactually invariant functions, and if these are in some cases `special exceptions' as opposed to something that happens in general.

\section{Conclusion}
In this paper we have examined numerous different existing versions of counterfactual invariance. From here our contributions were to analyse the relationship between counterfactual invariance and the causal graph, provide a variety of results on the implications between conditional independence and counterfactual invariance. Finally we showcased that for discrete models counterfactually invariant functions are almost always functions of a select few covariates. This creates difficulties for finding counterfactually invariant functions if many covariates causally depend on the intervening variable. 
\section{Acknowledgments}
The authors would like to thank Shahine Bouabid and Jean-François Ton for helpful discussions and valuable feedback. Jake Fawkes gratefully acknowledges funding from the EPSRC.
\bibliography{example_paper}
\bibliographystyle{icml2023}

\newpage
\appendix
\onecolumn
\section{Proofs}

\subsection{Definitions of Counterfactual Invariance}

\begin{customlemma}{\ref{lem:almos_sure_relation_dist}}
We have that \textit{a.s.}-CI implies $\mathcal{D}$-CI conditional on any set of variables, but $\mathcal{D}$-CI implies \textit{a.s.}-CI only if the conditioning set contains the outcome, $Y$.
\end{customlemma}

\begin{proof}
Firstly an almost sure equality of variables implies they are equal conditional on any event they must be almost surely equal conditional on almost all $w$ for any set $W$. Hence as almost sure equality implies any form of distributional equality.

Secondly we need to show that distributional counterfactual invariance does not imply almost sure counterfactual invariance, unless $W$ contains the outcome. First suppose $W$ does not contain the outcome, then we can generate $Y$ and $Z$ from the following model:
    \begin{align*}
        Z \sim Ber \left(\frac{1}{2}\right), 
        \quad U_Y \sim Ber \left(\frac{1}{2}\right), 
        \quad Y \defeq Z \oplus U_Y,
    \end{align*}
    where $Ber(p)$ represents a Bernoulli random variable with probability of success $p$, and $\oplus$ represents addition modulo 2;
all other variables are generated  independently. Then we have, conditional on any set $W$:
\begin{align*}
P( Y(z) \mid W=w, Z=z') &= P( Y(z) \mid Z=z^{\prime}) 
= \frac{1}{2}.
\end{align*}
So that $P( Y(z) \mid W=w, Z=z^{\prime}) = P( Y(z^{\prime}) \mid W=w, Z=z^{\prime})$ and so we have distributional counterfactual invariance. However we have $P( Y(0) = Y(1)) = 0$, so we do not satisfy almost sure counterfactual invariance. 

For the final statement we require that if the conditioning set $W$ contains the outcome variable $Y$ then both forms of counterfactual invariance are equivalent. For this we note that distributional equality implies:
    \begin{align*}
        P( {Y}(z) = y \mid W=w, Z=z) = P( {Y}(z^{\prime}) = y \mid W=w, Z=z),
    \end{align*}
but
\begin{align*}
    P( {Y}(z) = y \mid W=w, Z=z) &= P( {Y} = y \mid W=w, Z=z) \\
    &=P( {Y} = y \mid Y=y^{\prime} , Z=z) \\
    &= \mathbbm{1} \left\{ y =y^{\prime} \right\}.
\end{align*}
Therefore for almost all $w,z$ we have that:
\begin{align*}
P( \hat{Y}(z^{\prime}) = y^{\prime} \mid W=w, Z=z) = 1,
\end{align*}
implying that $P(Y(z) = Y(z^{\prime}) = 1 $ hence $\hat{Y}(z) \overset{a.s.}{=} \hat{Y}(z^{\prime})$.
\end{proof}

\begin{customcorollary}{\ref{lem:func_ci_equiv_dist}}
    A function $f$ is $\mathcal{F}$-CI if and only if $\hat{Y} = f(X) $ is $\mathcal{D}$-CI conditional $X$.
\end{customcorollary}
\begin{proof}
We already have that almost sure counterfactually invariance implies distributional counterfactual invariance, so all that remains to show is the other direction. For this we can apply the previous lemma, as if we condition on the inputs of $f$ we also condition on $\hat{Y}$. 
\end{proof}

\subsection{Relationship between Counterfactual Invariance and Edges in a Causal Graph}
\begin{customlemma}{\ref{lem:cf_invariance_no_edge}}
    If $Y$ is almost surely counterfactually invariant with respect to $Z$ then there is no edge from $Z$ to $Y$.
\end{customlemma}
\begin{proof}
    We use the definition of \citet{bongers2021foundations} which says that `$Z$ is not a parent of $Y$' is equivalent to there being a measurable function $\Tilde{f}_{Y}$ such that:
\begin{align*}
    f_{Y}(\mathbf{x},\mathbf{e}) =\Tilde{f}_{Y}(\mathbf{x}_{\setminus Z},\mathbf{e})
\end{align*}
    for almost all $\mathbf{e},\mathbf{x}$, where $f_{Y}(\mathbf{x},\mathbf{e})$ is the original structural equation for $Y$ and $\mathbf{x}_{\setminus Z}$ is the covariates $\mathbf{x}$ without $Z$. But this is equivalent to $Y$ being distributionally counterfactual invariant conditional on $\mathbf{e},\mathbf{x}_{\setminus Z}$ as we can set:
    \begin{align}
    \Tilde{f}_{Y}(\mathbf{x}_{\setminus Z},\mathbf{e}) = \EE[Y \mid \Pa(Y)\setminus Z,z^{\prime},\mathbf{e}] 
    \end{align}
For some fixed $z^{\prime}$.

First we can note that $ f_{Y}(\mathbf{x}_{\setminus Z},z^{\prime},\mathbf{e}) = \Tilde{f}_{Y}(\mathbf{x}_{\setminus Z},\mathbf{e})$ as we have specified all parts of the structural equation model for $Y$. From here we have:
\begin{align}
    \EE[Y \mid \Pa(Y)\setminus Z,z^{\prime},\mathbf{e}] &= \EE[Y(z') \mid \Pa(Y)\setminus Z,z^{\prime},\mathbf{e}] \\
    &= \EE[Y(z') \mid \Pa(Y)\setminus Z,z,\mathbf{e}] \\
    &= \EE[Y(z) \mid \Pa(Y)\setminus Z,z,\mathbf{e}] \\
    &= \EE[Y \mid \Pa(Y)\setminus Z,z,\mathbf{e}] \\
\end{align}
Where the third line follows as $Y(z') \perp Z \mid \Pa(Y), \mathbf{e}$ and then we apply distributional counterfactual invariance.
\end{proof}

\subsection{Relationship between Counterfactual Invariance and Conditional Independence}

\subsubsection{Invariance Implies Independence}

\begin{lemma}{\ref{lem:invar_implies_indep}}
  Suppose $Y$ is distributionally counterfactual invariant conditional on some set $W$ such that there is some $S \subset W$ that is a valid adjustment set for $(Z,Y)$. Then we have $Y \perp Z \mid S$.   
\end{lemma}
\begin{proof}
As $S$ is a valid adjustment set we have that $Y(z^{\prime}) \perp Z \mid S$. Now we have:

\begin{align*}
    P( Y =y \mid  S=s,Z=z^{\prime}) &= P( Y(z^{\prime}) =y \mid S=s,Z=z^{\prime}) \\
    & = P( Y(z^{\prime}) =y \mid  S=s,Z=z) \\
    & = \EE_{P(W \mid S=s,Z=z)}[P( Y(z^{\prime}) =y \mid W, Z=z) \\
    &= \EE_{P(W \mid S=s,Z=z)}[P( Y(z) =y \mid W, Z=z)] \\
    &= P( Y(z) =y \mid S=s, Z=z) \\
    &= P( Y =y \mid S=s, Z=z)
\end{align*}
Therefore as $P( Y =y \mid  S=s,Z=z^{\prime}) = P( Y =y \mid S=s, Z=z)$ we have $Y \perp Z \mid S$. 
\end{proof}

\begin{corollary}\label{cor:indep_implies_by_a.s_and_f}
The following holds;
\vspace{-0.6em}\begin{itemize}
\setlength\itemsep{-0.15em}
    \item If $Y$ is \textit{a.s.-CI} then $Y \perp Z \mid S$ for any valid adjustment set $S$ for $(Z,Y)$
    \item  If $f$ is $\mathcal{F}$\textit{-CI} then $f(X) \perp Z \mid S$ for any valid adjustment set $S$ for $(Z,f(X))$
\vspace{-0.6em}\end{itemize}
\begin{proof}
    These hold as any distribution satisfying almost sure counterfactual invariance will satisfy distributional counterfactual invariance conditional on any set $W$, therefore the result in Lemma \ref{lem:invar_implies_indep} can be applied for any valid adjustment set $S$.
\end{proof}
\end{corollary}

\subsubsection{But Independence Does Not Imply Invariance}
\begin{customprop} {\ref{prop:indep_not_imply_a.s}}
For almost sure counterfactual invariance we have that:
    \vspace{-0.6em}\begin{itemize}
            \setlength\itemsep{-0.15em}
        \item No set of independences imply any bound on $P(Y(z) = Y(z^{\prime}))$, even if $Y$ is binary.
        \item For discrete models, the counterfactually invariant structural causal models have Lebesgue measure zero within the models satisfying the set of independences in Corollary \ref{cor:indep_implies_by_a.s_and_f}. 
    \end{itemize}
\end{customprop}
\begin{proof}
For any set of variables $V$ we can simply let $Z,Y$ be as in example \ref{ex: binary_response_DAG} and generate all other variables independently. So we have:
\begin{align}
    Y = \Tilde{f}_Y(Z,\Tilde{N}_Y)
\end{align}
Where if $\Tilde{N}_Y$ is equal to $i$, $Y=\Tilde{f}_Y(Z,i)=f_i(Z)$. The distribution $P(\Tilde{N}_Y)$ over $\{0,1,2,3\}$ has probability vector which lies in:
\begin{align}
    \mathbf{a}(\lambda) = \begin{pmatrix}
        0 \\
        1-2p \\
        p \\
        p \\
    \end{pmatrix} +\lambda\begin{pmatrix}
        1 \\
        1 \\
        -1 \\
        -1 \\
    \end{pmatrix}
\end{align}
Where $P(Y(z) = Y(z^{\prime})) = \mathbf{a}(\lambda)_0+\mathbf{a}(\lambda)_1$. If we let $p=\frac{1}{2}$ we can see that $\lambda$ can vary between $[0,\frac{1}{2}]$ with $P(Y(z) = Y(z^{\prime}))$ taking all values in $[0,1]$. Therefore we have constructed an example which satisfies all possible independences however $P(Y(z) = Y(z^{\prime}))$ is unconstrained. 

For the second statement we use use a generalised version of the function response framework as is found in \citet{peters2017elements} or \citet{gresele2022causal} so that the observable distribution over all variables $V$ arises as:
\begin{align}
    P(V) = \sum_{\mathbf{e}} P((\mathbf{n})) \prod^n_{i=1} \mathbbm{1} \{v_i = f_{i, n_i}(\mathbf{pa}_i)\} , 
\end{align}
where $n_i$ indexes the possible functions from $\Pa(V_i)$ to $\mathcal{V}_i$.
Now we can relate any probability distribution $P((\mathbf{N}))$ with some vector $p \in \Delta^{m-1}$ in the probability simplex where $m = \lvert \mathcal{N} \rvert $ and the entries of the probability vector are $P( \mathbf{N} = \mathbf{n})$.

Now all the conditional independence statements will correspond to polynomial constraints on the vector $p$ as in \citet{drton2008lectures}. That is they will correspond to a set of functions $\{ g_i \}^k_{i=1}$ where $g_i: \Delta^{m-1} \to \mathbb{R} $ and $k$ is the total number of constraints where we require:

\begin{align}
    g_i(p) = 0
\end{align}

Now let $\mathcal{M}$ denote the submodel of $\Delta^{m-1}$ that lies in the algebraic variety generated by $\{ g_i \}^k_{i=1}$, that is the submodel that satisfies the conditional independence relationship. Now take a rational parameterisation of $\mathcal{M}$, so some $G: \Theta^d \to \mathcal{M}$ where $\Theta^d$ is some open set in $\mathbb{R}^d$. From previously almost sure counterfactual invariance corresponds to:
\begin{align}
    P( Y(z) = Y(z^{\prime})) = 1
\end{align}
is a linear function in $p \in \mathcal{M}$ as:
\begin{align}
    P( Y(z) = Y(z^{\prime})) = \sum_{\substack{\mathbf{n} \\ Y(z,\mathbf{n}) = Y(z^{\prime},\mathbf{n})}} P (\mathbf{N} = \mathbf{n}).
\end{align}
And is so a rational function in $\theta \in \Theta^d$. However from the previous example we know there exists a point $\theta_0 \in \Theta^d$ such that $G(\theta_0)$ is not counterfactually invariant. Therefore as this rational function is non-zero at some value it must be non zero for almost all $\theta \in \Theta^d$ \citep{lojasiewicz1964triangulation}. This follows as any polynomial is either zero everywhere or on a set of measure zero.

Hence amongst the set of models satisfying the conditional independence set (given by $\Theta^d$) almost all SCM's are not counterfactually invariant.
\end{proof}

\begin{customprop}{\ref{prop:classify_discrete_CI_F}}
   For a given causal model let $\mathcal{F}^{CI}$ be the set of counterfactually invariant functions $f: \mathcal{X} \to \mathcal{Y}$. Supposing all the variables are discrete we have that for almost all structural causal models:
   \begin{align}
      \mathcal{F}^{CI} = \{f : f \text{ is a function of $\Nd(Z)$ only} \} .
   \end{align}
Alternatively stated, the set of causal models that admit a counterfactually invariant function which takes a descendant of $Z$ as an input, has Lebesgue measure zero.
\end{customprop}
\begin{proof}
    As for previous proos for a given causal graph $\mathcal{G}$ all structural causal models can be associated with a point in the probability simplex $p \in \Delta^{m-1}$. Now we work with the Lebesgue measure over $\Delta^{m-1}$. 

    Now suppose $f:\mathcal{X} \to \mathcal{Y}$ is some function which depends on a covariate that is a descendant of $Z$. We will show that for almost all causal models this function is not counterfactually invariant. To do so we note that the counterfactual invariance of $f$ will again correspond to some polynomial in our probability distribution $p$. Hence as this polynomial is either zero everywhere or almost nowhere we simply have to find a point $p$ under which $f$ is not counterfactually invariant. 
    
    To do this note $f$ takes as input some variable $D$ which is a descendant of $Z$. Hence there must be two values $d,d^{\prime}$ such that for a fixed value $\mathbf{x}_{\setminus D}$ of the other covariates $X \setminus D$ the function $f$ changes value. Now we choose $p$ such that $V$ is a deterministic function of $Z$ that flips between $d,d^{\prime}$ and all other covariates are set to $\mathbf{x}\setminus D$. Therefore as we intervene on $Z$ we change the value of $f$ so under this structural causal model $f$ is not counterfactually invariant, hence $f$ is not counterfactually invariant under almost all structural causal models.
    
    Now we have that for each function $f$ which depends on some descendants of $Z$ it is counterfactually invariant on a set of measure zero. We can take the union of these sets over all functions to get the set of structural causal models for which there exists one counterfactually invariant function which depends on a descendant of $Z$. As this is the union of finitely many measure zero subsets it is also measure zero we have that almost everywhere the set of counterfactually invariant functions is:

    \begin{align}
      \mathcal{F}^{CI} = \{f : f \text{ is a function of $\Nd(Z)$ only.} \} 
   \end{align}
   
\end{proof}
\subsubsection{Distributional Invariance}\label{App:distributional_inv}
In the previous section we demonstrated that almost sure independence is not implied or even bounded by any set of independences. However distributional counterfactual invariance can be implied by conditional independence:

\begin{lemma}\label{lem:indep_can_imply_dist}
    If $W$ is a valid adjustment set then $Y \perp Z \mid W$ is equivalent to $Y$ being distributionally counterfactual invariant conditional on $W$.
\end{lemma}
\begin{proof}
    As $W$ is a valid adjustment set we have that $Y(z) \perp Z \mid W$. Therefore we may write:
    \begin{align}
        P(Y(z) \mid W=w, Z=z^{\prime}) &= P(Y(z) \mid W=w, Z=z) \\
        & =P(Y \mid W=w, Z=z)
    \end{align}
For any $z,z^{\prime}$. Therefore distributional counterfactual invariance simply corresponds to:
\begin{align}
   P(Y \mid W=w, Z=z^{\prime}) =P(Y \mid W=w, Z=z)
\end{align}
Which is exactly the conditional independence. 
\end{proof}

However this is a unique subset of distributional counterfactual invariance where the counterfactual distribution is identifiable from the observational distribution.  In general this is not possible and so we can recover similar results to the previous case:

\begin{proposition} \label{prop:indep_not_imply_dist}
    Suppose there is some vertex $w \in W$ such that $w \in \Pa(Y) \cap \De(Z)$ and $Y \not \perp w$. The for distributional counterfactual invariance conditional on $W$ we have all the results from almost sure invariance. That is:
    \begin{itemize}
        \item No additional independences imply distributional counterfactual invariance conditional on $W$.
        \item If $Y$ is binary no additional independences are sufficient to bound: 
        $$\hspace{-0.4cm} \lvert P(Y(z) \! = \! 1 | W \! = \!w, Z \! = \!z) -P(Y(z^{\prime})  \! = \! 1 \mid W \! = \!w, Z \! = \!z) \rvert.$$
        \item For discrete models, the distributional counterfactual invariant structural causal models have Lebesgue measure zero within the models satisfying any set of conditional independences.
    \end{itemize}
\end{proposition}

\begin{proof}
    The first follows as $Y$ is allowed to depend arbitrarily on $w$ we may simply set $Y =w$. Distributional counterfactual invariance conditional on $W$ then becomes almost sure counterfactual invariance of $w$ by lemma \ref{lem:almos_sure_relation_dist}. From here we can apply proposition \ref{prop:indep_not_imply_a.s} to show the first two results, noting that the only dependence implied so far is between $Y$ and $w$. This leads to the final result as distributional counterfactual invariance is again a selection of polynomials in our parameter vector.
\end{proof}
\end{document}